%% file: main.tex
\documentclass{article}

\newcommand{\tool}[0]{MSSR}

\usepackage{multirow}

\usepackage{microtype}
\usepackage{graphicx}
\usepackage{subcaption}
\usepackage{booktabs} % for professional tables

\usepackage{hyperref}

\usepackage[preprint]{icml2026}

\usepackage{amsmath}
\usepackage{amssymb}
\usepackage{mathtools}
\usepackage{amsthm}

\usepackage[capitalize,noabbrev]{cleveref}

\theoremstyle{plain}
\newtheorem{theorem}{Theorem}[section]
\newtheorem{proposition}[theorem]{Proposition}
\newtheorem{lemma}[theorem]{Lemma}

\theoremstyle{definition}

\theoremstyle{remark}

\usepackage[textsize=tiny]{todonotes}

\icmltitlerunning{MSSR: Memory-Aware Adaptive Replay for Continual LLM Fine-Tuning}

\begin{document}

\twocolumn[
  \icmltitle{MSSR: Memory-Aware Adaptive Replay for Continual LLM Fine-Tuning}



  \icmlsetsymbol{equal}{*}

  \begin{icmlauthorlist}
    \icmlauthor{Yiyang Lu}{CUHK}
    \icmlauthor{Yu He}{NJU}
    \icmlauthor{Jianlong Chen}{CUHK}
    \icmlauthor{Hongyuan Zha}{CUHK}
  \end{icmlauthorlist}

  \icmlaffiliation{CUHK}{The Chinese University of Hong Kong, Shenzhen}
  \icmlaffiliation{NJU}{Nanjing University}

  \icmlcorrespondingauthor{Yiyang Lu}{yiyanglu1@link.cuhk.edu.cn}

  \icmlkeywords{Machine Learning, ICML}

  \vskip 0.3in
]



\printAffiliationsAndNotice{}  % no special notice (required even if empty)

\begin{abstract}
\input{sections/0-abstract}

\end{abstract}

\input{sections/1-introduction}
\input{sections/2-preliminary}
\input{sections/3-methods}
\input{sections/4-experiments}
\input{sections/5-related_work}

\input{sections/6-conclusion}
\input{sections/impact_statement}
\input{sections/Acknowledgements}


\bibliography{ref_paper}
\bibliographystyle{icml2026}

%%%%%%%%%%%%%%%%%%%%%%%%%%%%%%%%%%%%%%%%%%%%%%%%%%%%%%%%%%%%%%%%%%%%%%%%%%%%%%%
%%%%%%%%%%%%%%%%%%%%%%%%%%%%%%%%%%%%%%%%%%%%%%%%%%%%%%%%%%%%%%%%%%%%%%%%%%%%%%%
% APPENDIX
%%%%%%%%%%%%%%%%%%%%%%%%%%%%%%%%%%%%%%%%%%%%%%%%%%%%%%%%%%%%%%%%%%%%%%%%%%%%%%%
%%%%%%%%%%%%%%%%%%%%%%%%%%%%%%%%%%%%%%%%%%%%%%%%%%%%%%%%%%%%%%%%%%%%%%%%%%%%%%%
\newpage
\appendix
\onecolumn
\input{sections/appendix}
%%%%%%%%%%%%%%%%%%%%%%%%%%%%%%%%%%%%%%%%%%%%%%%%%%%%%%%%%%%%%%%%%%%%%%%%%%%%%%%
%%%%%%%%%%%%%%%%%%%%%%%%%%%%%%%%%%%%%%%%%%%%%%%%%%%%%%%%%%%%%%%%%%%%%%%%%%%%%%%

\end{document}

%% file: sections/0-abstract.tex
Continual fine-tuning of large language models (LLMs) is becoming increasingly crucial as these models are deployed in dynamic environments where tasks and data distributions evolve over time. While strong adaptability enables rapid acquisition of new knowledge, it also exposes LLMs to catastrophic forgetting, where previously learned skills degrade during sequential training.
Existing replay-based strategies, such as fixed interleaved replay, accuracy-supervised, and loss-driven scheduling, remain limited: some depend on heuristic rules and provide only partial mitigation of forgetting, while others improve performance but incur substantial computational overhead. 
Motivated by retention dynamics under sequential fine-tuning, we propose Memory-Inspired Sampler and Scheduler Replay (\textbf{\tool}), an experience replay framework that estimates sample-level memory strength and schedules rehearsal at adaptive intervals to mitigate catastrophic forgetting while maintaining fast adaptation.
Extensive experiments across three backbone models and 11 sequential tasks show that \textbf{\tool} consistently outperforms state-of-the-art replay baselines, with particularly strong gains on reasoning-intensive and multiple-choice benchmarks.

%% file: sections/1-introduction.tex
\section{Introduction}
\label{sec:introduction}
\input{figures/strategy_model}
Large Language Models (LLMs) have demonstrated strong capabilities across a wide range of natural language processing tasks~\citep{qin2024large,hadi2023large,zhao2023survey,kaddour2023challenges,yang2024harnessing,zhuang2023through,team2024gemini,guo2025deepseek,chen2024internvl}. 
As these models are increasingly deployed in dynamic and evolving environments, there is a growing demand for continual learning (CL), enabling models to acquire new knowledge incrementally while retaining previously learned skills~\citep{wang2024comprehensive,chen2018lifelong}. 
This need is particularly evident in domains such as healthcare~\citep{lee2020clinical,amrollahi2022leveraging}, personalized applications~\citep{cai2022reloop,yang2025pcl}, and law and policy~\citep{chalkidis2021lexglue,zhang2023reformulating}. 
However, continual fine-tuning of LLMs remains challenging due to representation drift and gradient interference, often resulting in \textit{catastrophic forgetting}~\citep{van2024continual,luo2025empirical,zhai2023investigating,ren2024analyzing}.

Replay-based continual learning has been widely recognized as an effective strategy for mitigating catastrophic forgetting, with representative methods such as \textit{AQM}~\citep{caccia2020online}, \textit{GEM}~\citep{lopez2017gradient}, \textit{A-GEM}~\citep{chaudhry2018efficient}, and \textit{LOGD}~\citep{tang2021layerwise} primarily focusing on buffer construction and memory utilization. 
Beyond storage efficiency, prior work has also explored replay scheduling mechanisms, including fixed interleaving, accuracy-based replay~\citep{bang2021rainbow}, and loss-driven scheduling~\citep{aljundi2019online}. 
Despite this progress, existing replay strategies still exhibit notable limitations: 
(i) they are largely heuristic and lack grounding in cognitive memory theory, limiting principled scheduling decisions~\citep{murre2015replication}; 
(ii) they inadequately model the temporal heterogeneity of forgetting, often assuming uniform replay intervals across time scales; and 
(iii) their scalability to LLM fine-tuning remains unclear, as most evaluations focus on small-scale or short-horizon settings, while monitoring overhead becomes prohibitive in long training runs~\citep{ke2023continual}.

As illustrated in \autoref{fig:strategy-model}, we conceptually compare representative replay scheduling strategies. 
Fixed replay applies rehearsal at uniform intervals, while loss- and accuracy-based methods trigger replay based on performance signals. 
Although these strategies offer different trade-offs, they remain reactive and lack a principled mechanism to align replay timing with the evolving forgetting dynamics of the model.

To address these limitations, we propose \tool, a continual fine-tuning framework for LLMs inspired by the Ebbinghaus forgetting curve. 
Rather than relying on fixed or reactive replay triggers, \tool\ models memory retention as a time-dependent decay process and schedules replay accordingly, progressively expanding replay intervals as model stability increases.
This design provides a cognitively motivated yet practical alternative to heuristic replay strategies for long-horizon LLM continual learning.

To rigorously establish the validity of our contributions, we organize our work into three interrelated stages. We first propose an memory-inspired replay scheduling framework that bridges cognitive memory theory with continual learning for LLMs. Building on this design, we further introduce a methodological perspective that highlights how cognitively motivated scheduling can serve as a principled alternative to existing heuristic strategies\citep{howard2002distributed,kaplan2020scaling}. Finally, we validate the effectiveness of our approach through extensive experiments on reasoning benchmarks such as GSM8K\citep{cobbe2021gsm8k}, MATH\citep{hendrycks2021measuring}, and MMLU\citep{hendrycks2021ethics}, showing that it achieves favorable retention–efficiency trade-offs and substantially mitigates forgetting with minimal computational overhead.
In summary, our contributions are as follows:

\begin{itemize}
    \item \textbf{Framework.} We introduce a memory-aware adaptive replay sampler and scheduler that bridges cognitive memory theory and continual learning in LLMs.
    \item \textbf{Methodological insight.} We show how cognitively motivated scheduling provides a principled alternative to existing heuristic replay strategies.
    \item \textbf{Empirical experiments.} We demonstrate through experiments on long-sequence reasoning benchmarks (e.g., GSM8K, MATH, MMLU) that our approach improves retention–efficiency trade-offs and mitigates catastrophic forgetting with minimal overhead.
\end{itemize}

%% file: figures/strategy_model.tex
\begin{figure}[t]
  \centering
  % 第1行
  \begin{subfigure}[t]{0.49\columnwidth}
    \centering
    \includegraphics[width=\linewidth]{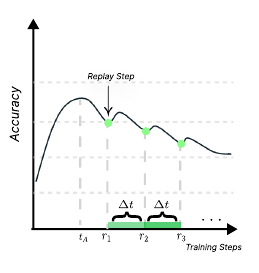}
    \caption{Fixed Replay}
    \label{fig:fixed}
  \end{subfigure}\hfill
  \begin{subfigure}[t]{0.49\columnwidth}
    \centering
    \includegraphics[width=\linewidth]{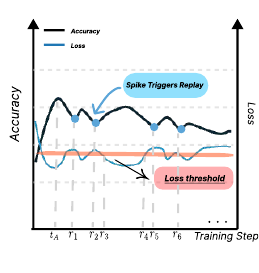}
    \caption{Loss-Based Replay}
    \label{fig:loss}
  \end{subfigure}

  \vspace{2mm} % 行间距，可调

  % 第2行
  \begin{subfigure}[t]{0.49\columnwidth}
    \centering
    \includegraphics[width=\linewidth]{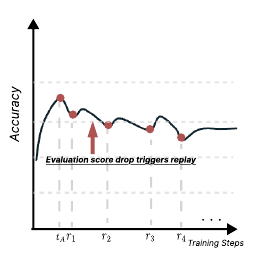}
    \caption{Accuracy-Based Replay}
    \label{fig:accuracy}
  \end{subfigure}\hfill
  \begin{subfigure}[t]{0.49\columnwidth}
    \centering
    \includegraphics[width=\linewidth]{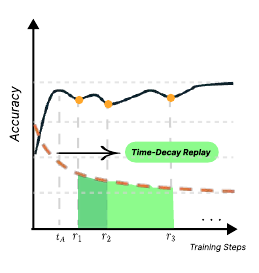}
    \caption{MSSR (Ours)}
    \label{fig:mssr}
  \end{subfigure}

  \caption{
Comparison of replay triggering strategies in continual fine-tuning.
(a) \textbf{Fixed replay} performs replay at a constant interval, ignoring optimization dynamics.
(b) \textbf{Loss-based replay} triggers replay when the loss exceeds a threshold, but noisy high-frequency fluctuations can cause frequent spurious triggers.
(c) \textbf{Accuracy-based replay} reacts to evaluation drops, yet often suffers from lag since replay starts after accuracy has already degraded.
(d) \textbf{MSSR (ours)} is time-aware and memory-inspired, scheduling replay based on time-dependent retention to stabilize long-term performance.
}
  \label{fig:strategy-model}
\end{figure}

%% file: sections/2-preliminary.tex
\section{Preliminaries}

\paragraph{Catastrophic Forgetting Measurement}
In continual learning or multi-stage fine-tuning, the model is sequentially trained on a series of datasets
$\{D_1, D_2, \dots, D_T\}$, each corresponding to a task or domain.
Let $F_t(D_i)$ denote the performance (e.g., accuracy or loss) of the model after completing the training on dataset $D_t$
and evaluated on the previously seen dataset $D_i$ $(i \le t)$.
The average forgetting is defined as
\begin{equation}
    \mathcal{F} = \frac{1}{T-1}
    \sum_{i=1}^{T-1}
    \max_{t > i} [F_i(D_i) - F_t(D_i)],
    \label{eq:forgetting}
\end{equation}
where a larger $\mathcal{F}$ indicates more severe forgetting.
This metric quantifies the degradation of model performance on previous tasks as new tasks are learned.

\paragraph{Experience Replay Formalization}
Experience replay mitigates catastrophic forgetting by mixing old and new data during fine-tuning.
Let $D_{\text{new}}$ and $B_{\text{replay}}$ denote the current training data and the replay buffer containing representative samples from previous datasets, respectively.
At each training step, the model is optimized over a mixed dataset:
\begin{equation}
    D_{\text{mix}} = D_{\text{new}} \cup \{(x_i, y_i) \sim B_{\text{replay}}\}.
\end{equation}
The total training objective is
\begin{equation}
\begin{split}
    \mathcal{L}_{\text{total}}
    = &~ \mathbb{E}_{(x,y)\sim D_{\text{new}}}
        \!\big[\ell(f_\theta(x),y)\big] \\[2pt]
      &~ + \lambda~
        \mathbb{E}_{(x,y)\sim B_{\text{replay}}}
        \!\big[\ell(f_\theta(x),y)\big],
\end{split}
\label{eq:replay_loss}
\end{equation}
where $\ell(\cdot)$ is the task loss, $\lambda$ controls the replay ratio,
and $\theta$ denotes model parameters.
By periodically re-exposing the model to past samples,
experience replay alleviates parameter drift and knowledge degradation.

\paragraph{Forgetting Curve as Scheduling Inspiration}
The Ebbinghaus forgetting curve characterizes memory retention as a monotonically decreasing function of elapsed time,
with repeated reviews progressively slowing down the decay and extending optimal review intervals.
In this work, we adopt this principle as a \emph{heuristic inspiration} rather than a literal cognitive model.
Specifically, it motivates time-dependent replay scheduling and memory-aware prioritization,
without assuming that large language models follow the same forgetting dynamics as humans.

%% file: sections/3-methods.tex
\section{\tool: Modeling Memory Decay for Time-Dependent Replay}
\label{sec:modeling}
\input{figures/framework}

\subsection{Sample-Level Memory Strength Modeling}
\label{method: sample_modeling}

We model the retention of each sample $i$ as a stochastic decay process modulated by time and difficulty, 
following classical forgetting and survival-based memory models~\citep{ebbinghaus1885gedachtnis, rubin1996one, wixted2004psychology}. 
Let $t\!\in\!\mathbb{N}$ denote training steps and $\mathcal{R}_i$ the set of replay exposures. 
We define a \emph{memory strength} $m_{i,t}\!\in\!(0,1]$ and a \emph{stability} variable $S_{i,t}\!>\!0$ controlling resistance to forgetting:
\begin{equation}
\begin{aligned}
m_{i,t+1} &= m_{i,t}\exp(-h_{i,t}), \\
h_{i,t} &= \frac{\alpha_i+\gamma_d\,\phi(\widehat{\ell}_{i,t})}{S_{i,t}},
\end{aligned}
\label{eq:retention_update}
\end{equation}
where $\alpha_i$ denotes baseline decay, $\gamma_d$ controls loss sensitivity, 
and $\phi$ is a monotone mapping (e.g., calibrated sigmoid) applied to normalized loss $\widehat{\ell}_{i,t}$.\footnote{In practice, $\widehat{\ell}_{i,t}$ is computed via EMA-denoised loss followed by quantile normalization; see the Appendix for details.}

\paragraph{Review and Consolidation.}
At each review step $t\!\in\!\mathcal{R}_i$, the memory state is reset and stabilized through:
\begin{equation}
\begin{split}
m_{i,t^{+}} &= 1, \\[3pt]
S_{i,t^{+}} &= S_{i,t}
+ \eta_s (S_{\max} - S_{i,t})^{\beta} e^{-\rho \Delta t_i} \\[-2pt]
&\quad\;\times (1 - m_{i,t})^{\gamma_s}
+ \epsilon_t,
\end{split}
\label{eq:reset_spacing}
\end{equation}

where $\Delta t_i = t - t_i^{\star}$ is the elapsed time since the last review, 
and $\epsilon_t \!\sim\! \mathcal{N}(0,\sigma_s^2)$ captures stochastic variation.
Here, $\eta_s$, $\beta$, $\rho$, and $\gamma_s$ respectively control learning rate, saturation, spacing sensitivity, and error-driven reinforcement.
This rule jointly models saturating stability growth, spacing modulation, error-dependent consolidation, and stochastic replay efficiency.

\paragraph{Epoch-Level Update.}
Unrolling Eq.~\eqref{eq:retention_update} gives:
\begin{equation}
m_{i,t}\approx\exp\!\Big(-\!\!\int_{t_i^{\star}}^{t}\!h_i(\tau)d\tau\Big),
\label{eq:discrete_survival}
\end{equation}
where $t_i^{\star}$ is the last review step. 
To reduce computation in large-scale fine-tuning, we use a piecewise-constant hazard updated at epoch boundaries 
$\{T_0, T_1, \dots, T_E\}$ with $\Delta t_e = T_e - T_{e-1}$:
\begin{equation}
\begin{aligned}
m_{i,T_e}=m_{i,T_{e-1}}e^{-h_{i,T_e}\Delta t_e}, \\
h_{i,T_e}=\frac{\alpha_i+\gamma_d\,\phi(\widehat{\ell}_{i,T_e})}{S_{i,T_e}}.
\end{aligned}
\label{eq:epoch_update}
\end{equation}
This provides an efficient discrete approximation to the continuous retention dynamics suitable for large-scale LLM fine-tuning.

\subsection{Scheduling Dataset-Level Replay Dynamics}
Building on the sample-level memory formulation, 
we now model replay scheduling at the dataset level, 
determining \emph{when} replay occurs and \emph{how much} past data is mixed with new samples at each step. 
This process captures the temporal organization of rehearsal events, 
analogous to the spaced reinforcement observed in human memory systems.

\paragraph{Replay timing and spacing expansion.}
Let $t\!\in\!\mathbb{N}$ denote training steps, and 
$\mathcal{T}_r=\{t_1, t_2, \dots\}$ be the sequence of replay checkpoints. 
Following the spacing principle, the interval between two adjacent replay events expands gradually:
\begin{equation}
\Delta t_{r}^{(k+1)} = \Delta t_{r}^{(k)} (1 + \eta_p\, e^{-\rho_p k}),
\label{eq:interval_expand}
\end{equation}
where $\Delta t_{r}^{(k)} = t_{k} - t_{k-1}$ is the $k$-th replay interval, 
$\eta_p$ controls the expansion rate, and $\rho_p$ regulates how quickly the spacing growth saturates.
This ensures that replay events are dense in the early stage to prevent rapid forgetting and gradually sparser as memory stabilizes.
(See Appendix~\ref{app:interval_expansion} for derivation from the stability dynamics.)
\paragraph{Dynamic replay ratio and composition.}
At each replay step $t_k\!\in\!\mathcal{T}_r$, 
we construct a mixed batch combining current-task data $D_{\text{new}}$ 
and replayed data $B_{\text{replay}}$ from previous tasks:
\begin{equation}
\begin{aligned}
D_{\text{mix}}^{(t_k)} = 
D_{\text{new}} \cup 
\{(x_i, y_i)\sim B_{\text{replay}}\},
\\
\lambda_{t_k} = 
\lambda_0\, e^{-\beta_r t_k} + \lambda_{\min},
\end{aligned}
\label{eq:dynamic_ratio}
\end{equation}
where $\lambda_{t_k}$ controls the ratio of replayed samples within the mixed batch. The exponential decay form arises naturally from an optimal control perspective balancing rehearsal benefit and computational cost 
(Appendix~\ref{app:ratio_derivation}). Initially, $\lambda_{t_k}\!\approx\!\lambda_0$ encourages strong rehearsal, 
while the ratio decays exponentially toward $\lambda_{\min}$ 
as training progresses and model stability increases.

\paragraph{Interaction with sample-level retention.}
Each replayed sample $i$ maintains its own memory strength $m_{i,t}$,
which determines its forgetting risk. 
To focus replay on more unstable samples, 
we weight each sample's replay probability by a normalized inverse retention score:
\begin{equation}
p_i^{(t_k)} = 
\frac{m_{i,t_k}^{-\zeta}}
{\sum_{j\in B_{\text{replay}}} m_{j,t_k}^{-\zeta}},
\label{eq:prob_weight}
\end{equation}
where $\zeta\!>\!0$ controls the prioritization intensity. 
This design extends memory-based prioritization~\citep{schaul2015prioritized}
to continual fine-tuning: samples with lower $m_{i,t}$ (faster forgetting)
are replayed more often. 
A full derivation linking this weighting to the expected consolidation gain
is provided in Appendix~\ref{app:prob_derivation}.

\subsection{Integration Components: the \tool\ Framework}
\paragraph{Framework Overview.}
To operationalize the sample-level and dataset-level formulations introduced above, we develop a unified continual-learning framework, termed \tool. 
Built upon the \textbf{LoRA}-based fine-tuning pipeline \citep{hu2022lora} implemented in \textbf{LLaMAFactory} \citep{zheng2024llamafactory}, \tool\ transforms the theoretical retention dynamics into a practical replay-driven training algorithm.
As illustrated in Fig.~\ref{fig:framework}, \tool\ implements a memory-aware continual fine-tuning pipeline organized as a closed-loop workflow.
On the left, a sample memory module tracks per-sample memory strength $m_{i,t}$, which is updated at the epoch level based on observed loss and time-dependent decay. These memory states are converted into replay probabilities via probabilistic sampling, prioritizing samples with weaker memory or longer unseen intervals.
On the right, an adaptive replay scheduler determines both when to trigger replay, using an expanding-interval strategy, and how many samples to replay, through a time-decaying replay ratio. The selected replay set is mixed with the current task data and jointly optimized using LoRA-based fine-tuning.
Together, these components provide a unified framework for mitigating catastrophic forgetting in sequential learning.

\paragraph{Sample Memory Tracking.}
This module maintains per-sample retention states during fine-tuning.
At each epoch, the observed loss $\widehat{\ell}_{i,t}$ updates the memory strength and stability according to Eq.~\eqref{eq:retention_update}, capturing time-dependent forgetting and difficulty-dependent adaptation.
The resulting states $\{m_{i,t}, S_{i,t}\}$ provide replay-related statistics that guide the subsequent scheduler.

\paragraph{Replay Scheduler.}
This module converts the tracked memory states into adaptive replay decisions. 
Each sample is assigned a replay probability
$p_i^{\text{replay}} \propto (1 - m_{i,t})^{\beta_m} e^{\rho \Delta t_i}$,
favoring items with weaker memory or longer unseen intervals.
An adaptive replay ratio $r_t$ further controls the overall replay volume.
The selected replay set $\mathcal{R}_t$ is then mixed with new samples for LoRA-based fine-tuning, enabling memory-aware continual learning.

\paragraph{LoRA-based Fine-tuning Integration.}
After the replay subset $\mathcal{R}_t$ is selected, it is merged with the current task data $\mathcal{D}_t$ for joint optimization under the parameter-efficient fine-tuning scheme introduced above.
During each training step, the combined dataset $\mathcal{D}_t \cup \mathcal{R}_t$ is used to compute the total loss:
\begin{equation}
\mathcal{L}_t = 
\mathbb{E}_{(x,y)\sim \mathcal{D}_t \cup \mathcal{R}_t}
    \big[\,\ell(f_{\theta}(x), y)\,\big],
\label{eq:lora_loss}
\end{equation}
where $f_{\theta}$ denotes the backbone model equipped with LoRA adapters and $\ell(\cdot)$ is the task-specific loss.
Gradients from both newly introduced and replayed samples are aggregated to update the LoRA parameters $\theta_{\text{LoRA}}$ via backpropagation.
This joint training mechanism enables \tool\ to continuously refine task-relevant knowledge while mitigating forgetting, leveraging replay without increasing the total number of trainable parameters.

\paragraph{Training Workflow.}
As shown in Algorithm~\ref{alg:tool}, the overall training of \tool\ forms an end-to-end continual learning loop integrating memory updates, replay scheduling, and LoRA-based optimization. 
At initialization, all samples start with $m_{i,0}=1$ and default stability $S_{i,0}$. 
In each epoch, sample losses update memory states via Eq.~\eqref{eq:retention_update}; 
the scheduler then selects a replay subset $\mathcal{R}_t$ by Eq.~\eqref{eq:prob_weight}; 
and the model is optimized on $\mathcal{D}_t \cup \mathcal{R}_t$ using the joint loss in Eq.~\eqref{eq:lora_loss}. 
Afterward, memory statistics are recorded for monitoring and analysis.

%% file: figures/framework.tex
\begin{figure*}[t]
    \centering
    \includegraphics[width=\textwidth]{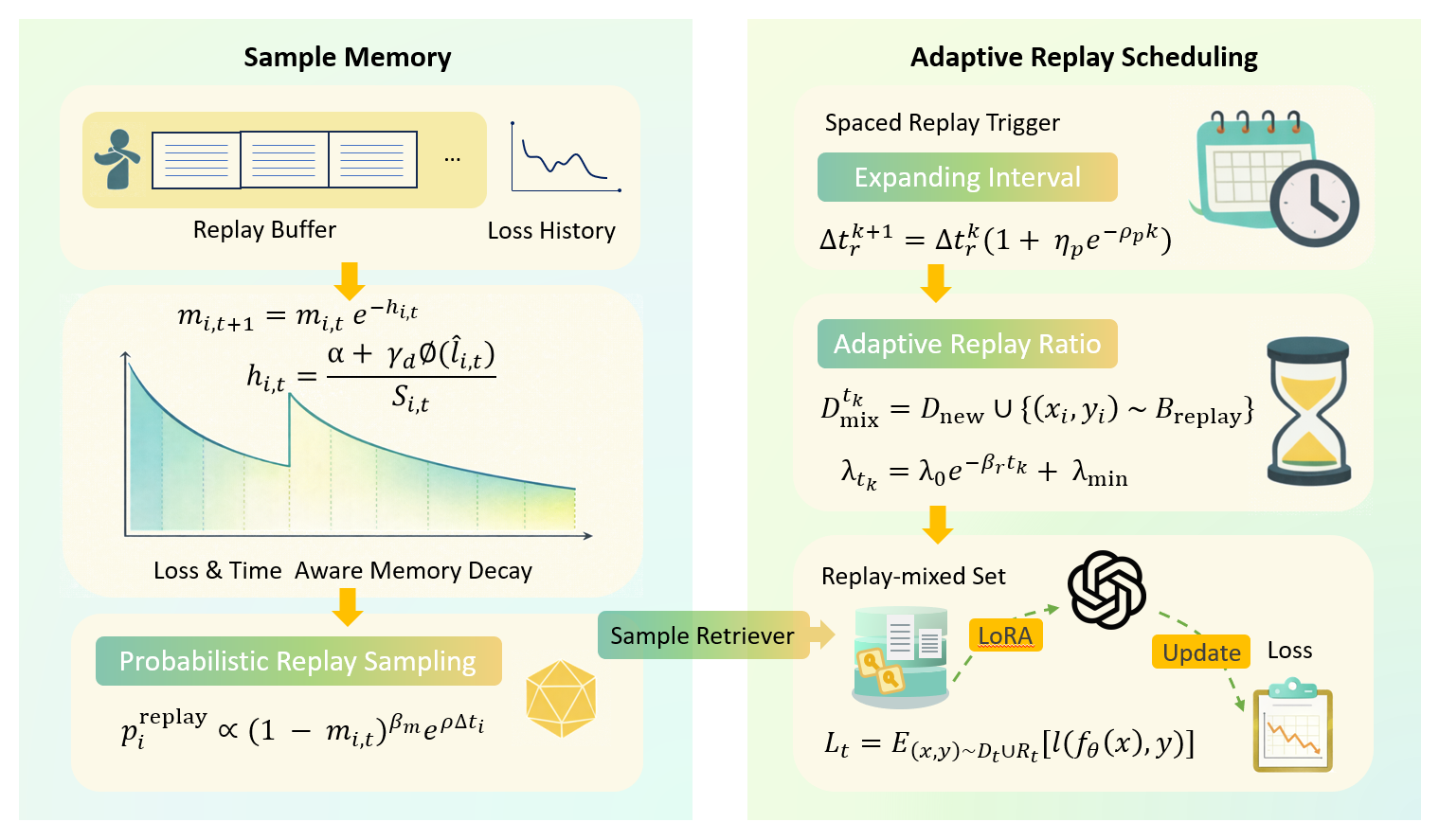}
    \caption{
        \textbf{Overall architecture of the \tool\ framework.}
        The framework consists of two core components that jointly govern replay behavior.
        (1) a sample-level replay sampler (left), which tracks per-sample memory strength by modeling
        loss-driven and time-dependent decay, and converts memory states into probabilistic replay weights.
        (2) an adaptive replay scheduler (right), which which regulates replay timing via expanding intervals and replay volume via a time-decaying ratio.
        At each replay event, sampled data are merged with current-task samples and used for LoRA-based
        fine-tuning, forming a closed-loop, memory-aware continual learning process.
    }
    \label{fig:framework}
\end{figure*}

%% file: sections/4-experiments.tex
\section{Experiments}
\subsection{Experimental Setup}

\paragraph{Tasks and Datasets.}
To rigorously evaluate catastrophic forgetting under different replay strategies, we adopt a sequential multi-domain fine-tuning setup. Our primary evaluation employs a progression of three datasets of increasing reasoning complexity, trained in the order of \textbf{Alpaca-GPT4}~\citep{peng2023instruction} (general instruction following) $\rightarrow$ \textbf{GSM8K-RFT}~\citep{cobbe2021gsm8k} (elementary mathematical reasoning) $\rightarrow$ \textbf{Competition Math}~\citep{hendrycksmath2021} (advanced problem solving). This setup allows us to examine knowledge interference between heterogeneous domains and retention when tasks share closer reasoning structures.

To further assess the scalability and robustness of our method in long-sequence continual learning, we extend the evaluation to an \textbf{11-task sequence} encompassing diverse domains such as AGNews, SQuAD, SciQ, BoolQ, ARC, and multiple MATH subsets. Detailed dataset descriptions and preprocessing procedures are provided in Appendix~\ref{app:training_datasets}.
\input{figures/table-training_datasets}

\paragraph{Evaluation Metrics.}
We evaluate the model after each training stage on the held-out test splits of the corresponding datasets. 
For SQuAD, we report the standard token-level F1 score following the official evaluation protocol.
For other tasks including mathematical reasoning datasets (GSM8K-RFT and Competition Math), we report \textbf{exact-match accuracy}. Additionally, to comprehensively evaluate the model's general knowledge retention amidst sequential fine-tuning, we consistently report its performance on the MMLU benchmark after each stage. An \emph{average normalized score} aggregated across all tasks is also computed to capture global retention stability. Formal metric definitions are provided in Appendix~\ref{app:eval_metric}.

\paragraph{Baselines.}
We compare several continual fine-tuning strategies:
(1) \textbf{No Replay}: Sequential training without reusing previous data.
(2) \textbf{Fixed Replay}: Uniformly replaying a constant subset of prior samples at each stage.
(3) \textbf{Loss-based Replay}: Triggering replay based on increased loss landscape sharpness/variance.
(4) \textbf{Accuracy-based Replay}: Activating replay upon detecting significant performance drops on previous validation sets.
(5) \textbf{\tool}: Our proposed method, where \tool\textsubscript{spl} uses the Ebbinghaus-inspired sampler, \tool\textsubscript{sch} uses the Ebbinghaus-inspired scheduler, and \tool\textsubscript{full} integrates both components.

\paragraph{Model Training.}
We perform all fine-tuning experiments using the \textbf{LoRA} framework~\citep{hu2022lora} as implemented in the \textbf{LLaMA-Factory} toolkit~\citep{zheng2024llamafactory}. We primarily use the \textbf{Qwen2.5-7B} model~\citep{yang2024qwen2} as the base architecture for detailed analysis, with additional experiments on \textbf{Gemma2-9B}~\citep{team2024gemma}, \textbf{Llama-3.1-8B}~\citep{dubey2024llama} and \textbf{Mistral-7B-v0.3}~\citep{jiang2023mistral7b} to ensure generalizability. Training is conducted in a distributed multi-GPU environment with NVIDIA A100 GPUs (80 GB each). For each dataset, fine-tuning proceeds until validation performance converges, with the total number of optimization steps kept consistent across all replay strategies to ensure fair comparison. All experiments use identical optimizer configurations and random seeds for reproducibility. Hyperparameter settings are detailed in Appendix~\ref{app:training_details}.

\subsection{Main Results}
We report the performance of different replay strategies under the standard 3-task setting in Table~\ref{tab:main_results}, 
and further evaluate their robustness under extended 11-task continual learning in Table~\ref{tab:longseq_all}. 
Several key observations emerge.

\textbf{(1) MSSR$_{\text{full}}$ achieves the strongest and most consistent performance across models and tasks.}  
Across both tables, MSSR$_{\text{full}}$ attains the best results on the majority of datasets and backbones, indicating that jointly modeling sample-level memory strength and adaptive replay scheduling yields clear benefits.  
MSSR$_{\text{sch}}$ and MSSR$_{\text{spl}}$ also consistently outperform baseline methods, frequently achieving the second-best performance across tasks, demonstrating the effectiveness of both components individually.

\textbf{(2) Sample-level and schedule-level variants exhibit complementary trade-offs.} 
In comparison, MSSR$_{\text{spl}}$ tends to outperform MSSR$_{\text{sch}}$ on a larger number of tasks, reflecting the advantage of fine-grained, sample-level replay prioritization.
However, MSSR$_{\text{sch}}$ requires lower computational overhead, as it avoids intensive per-sample operations.  
As a result, both variants offer practical advantages depending on resource constraints, while their combination in MSSR$_{\text{full}}$ provides the most robust overall performance.

\textbf{(3) Accuracy-based replay is competitive but computationally expensive.}  
The accuracy-based baseline achieves strong results on several tasks, 
often comparable to MSSR variants. However, it relies on frequent evaluation to trigger replay, leading to substantially higher computational and time costs.  
In contrast, MSSR attains similar or better performance without repeated evaluations, making it a more efficient and scalable solution.

\textbf{(4) MSSR is particularly effective at mitigating early-task forgetting in long sequences.} 
In the 11-task setting (Table~\ref{tab:longseq_all}), MSSR$_{\text{full}}$ achieves the best performance on most early tasks (e.g., the first six datasets), while the second-best results are almost always obtained by MSSR$_{\text{sch}}$ or MSSR$_{\text{spl}}$. This pattern highlights the strength of MSSR in alleviating catastrophic forgetting of earlier tasks as the task horizon grows.

\textbf{(5) Gains are task-dependent and most pronounced on moderately difficult benchmarks.}  
MSSR variants yield particularly large improvements on ARC, a multiple-choice reasoning benchmark, with gains of up to \textbf{+0.108} compared to baselines.  
This suggests that MSSR is especially effective when pre-trained LLMs exhibit low initial accuracy but can benefit from targeted rehearsal.  
It also achieves strong performance on \textbf{SQuAD}.
In contrast, for simpler tasks such as MATH$_1$, performance differences across methods are smaller, likely because these skills are already well captured during pre-training.

\subsection{Ablation Studies}
We analyze the influence of key hyperparameters, replay strategies, and memory buffer sizes on the overall replay behavior and continual learning performance of MSSR.

\paragraph{Replay Ratio.}
In our formulation (Eq.~\ref{eq:dynamic_ratio}), the replay ratio $\lambda_{t_k}$ is dynamically determined at each replay step rather than fixed a priori.  
To examine how this dynamic adjustment affects model performance, we simulate different initial replay coefficients $\lambda_0$ while keeping the decay schedule constant.  
Table~\ref{tab:ratio_ablation} reports the final test accuracy for representative tasks (Qwen2.5-7B, 11-task sequence).  

As shown, MSSR maintains stable performance over a wide range of $\lambda_0$, with best results around 0.10--0.20. This indicates that the method is robust to the choice of initial replay ratio, providing consistent accuracy without fine-tuning hyperparameters excessively.

\paragraph{Buffer Size.}
We next study the effect of buffer size, which determines the memory budget available for storing past samples. Smaller buffers exacerbate forgetting, while excessively large buffers may introduce redundant storage.  
Table~\ref{tab:buffer_size} reports average retention under different buffer sizes with the Ebbinghaus scheduler.  

Retention improves monotonically with buffer size, but marginal gains diminish beyond 2048 samples. This suggests that MSSR remains effective even under constrained memory budgets, thanks to cognitive-inspired scheduling.

\paragraph{Scheduler Variants.}
We evaluate the impact of different replay interval sequences to assess the robustness of scheduling strategies.  
Specifically, we compare fixed uniform intervals, geometric sequences, and the Ebbinghaus-inspired sequence.  
Results are shown in Table~\ref{tab:scheduler_variants}. All expanding-interval strategies outperform fixed spacing, with the Ebbinghaus sequence consistently achieving superior long-term retention. This confirms that cognitive-inspired scheduling provides a principled advantage over uniform or heuristic patterns.

\paragraph{Computational and Memory Overhead.}
Although MSSR reduces redundant replay, it requires maintaining per-sample memory strength and scheduling logic.  
Table~\ref{tab:overhead} reports normalized wall-clock time, peak memory, throughput, and per-step latency relative to a fixed-replay baseline (7B backbone). The overhead is minimal (3–5\% wall-clock, 4–6\% peak memory), as all updates involve only scalar operations per sample. No additional forward/backward passes are required, and throughput remains largely unchanged. Considering that MSSR improves accuracy by 1–3 points and reduces forgetting, this small computational cost is justified.

%% file: figures/table-training_datasets.tex
\begin{algorithm}[tb]
  \caption{Training workflow of the \tool\ framework}
  \label{alg:tool}
  \begin{algorithmic}
    \STATE {\bfseries Input:} dataset $\mathcal{D}$, initial parameters $\theta$, memory states $\{m_{i,0}, S_{i,0}\}$ 
    \FOR{each epoch $t = 1, \dots, T$}
        \FOR{each batch $(x_i, y_i)$ in $\mathcal{D}_t$}
            \STATE Compute loss $\widehat{\ell}_{i,t} = \ell(f_{\theta}(x_i), y_i)$
            \STATE Update $m_{i,t}, S_{i,t}$ according to Eq.~\eqref{eq:retention_update}
        \ENDFOR
        \STATE Estimate replay probabilities $p_i^{\text{replay}}$ using Eq.~\eqref{eq:prob_weight}
        \STATE Select replay set $\mathcal{R}_t$ based on $p_i^{\text{replay}}$ and ratio $r_t$
        \STATE Optimize LoRA parameters $\theta_{\text{LoRA}}$ on $\mathcal{D}_t \cup \mathcal{R}_t$ with loss Eq.~\eqref{eq:lora_loss}
        \STATE Record metrics $\{m_{i,t}, S_{i,t}, \widehat{\ell}_{i,t}\}$ for analysis
    \ENDFOR
  \end{algorithmic}
\end{algorithm}

\begin{table*}[t]
\centering
\caption{Basic 3-task continual learning results. Best results are shown in bold and second-best are underlined within each backbone.}
\label{tab:main_results}
\resizebox{\textwidth}{!}{
\begin{tabular}{l *{12}{c}}
\toprule
\multicolumn{1}{c}{Model} & \multicolumn{4}{c}{Mistral-7B-v0.3} & \multicolumn{4}{c}{Llama-3.1-8B} & \multicolumn{4}{c}{Qwen2.5-7B} \\
\addlinespace[0.5ex]
\cmidrule(r){1-1} 
\cmidrule(lr){2-5} \cmidrule(lr){6-9} \cmidrule(l){10-13}
Dataset & \multicolumn{1}{c}{MMLU} & \multicolumn{1}{c}{GSM8K} & \multicolumn{1}{c}{MATH} & \multicolumn{1}{c}{AVG} & \multicolumn{1}{c}{MMLU} & \multicolumn{1}{c}{GSM8K} & \multicolumn{1}{c}{MATH} & \multicolumn{1}{c}{AVG} & \multicolumn{1}{c}{MMLU} & \multicolumn{1}{c}{GSM8K} & \multicolumn{1}{c}{MATH} & \multicolumn{1}{c}{AVG} \\
\midrule
Vanilla              & 60.2 & 57.7 & 27.5 & 48.5 & 64.5 & 63.8 & 29.3 & 52.5 & 57.6 & 70.1 & 28.6 & 52.1 \\
R$_{\text{fixed}}$   & 60.5 & 58.8 & 28.3 & 49.2 & 65.1 & 64.4 & 29.5 & 53   & 58.5 & 70.5 & 29.3 & 52.8 \\
R$_{\text{loss}}$    & 60.7 & 58.4 & 28.8 & 49.3 & \underline{65.3} & 64.7 & 29.6 & 53.2 & 58.4 & \underline{72.3} & 30.4 & 53.7 \\
R$_{\text{accu}}$    & 61.0 & 59.3 & \underline{29.4} & \underline{49.9} & 64.9 & 64.5 & 29.9 & 53.1 & \textbf{59.5} & 71.7 & 29.9 & 53.7 \\
MSSR$_{\text{sch}}$  & 60.8 & \underline{59.5} & 28.6 & 49.6 & 64.8 & \underline{65.3} & 29.5 & 53.2 & 59.1 & 71.9 & \underline{31.1} & \underline{54.3} \\
MSSR$_{\text{spl}}$  & \underline{61.3} & 59.1 & 29.3 & \underline{49.9} & 65.1 & 64.7 & \underline{30.3} & \underline{53.4} & 58.9 & \underline{72.3} & 30.2 & 53.8 \\
MSSR$_{\text{full}}$ & \textbf{61.9} & \textbf{60.4} & \textbf{30.1} & \textbf{50.8} & \textbf{65.6} & \textbf{66.2} & \textbf{31.4} & \textbf{54.4} & \underline{59.2} & \textbf{72.5} & \textbf{31.7} & \textbf{54.5} \\
\bottomrule
\end{tabular}
}
\end{table*}

\begin{table*}[t]
\centering
\caption{Extended 11-task continual learning results. Best results are shown in bold and second-best are underlined within each backbone.}
\label{tab:longseq_all}
\resizebox{\textwidth}{!}{
\begin{tabular}{llccccccccccc}
\toprule
\textbf{Model} & \textbf{Method} 
& AGNews & SQuAD & SciQ & BoolQ & ARC & GSM8K 
& MATH$_1$ & MATH$_2$ & MATH$_3$ & MATH$_4$ & MATH$_5$ \\
\midrule

\multirow{7}{*}{Gemma2-9B}
& None 
& 0.714 & 73.98 & 0.952 & 0.889 & 0.336 & 0.520 
& 0.778 & 0.688 & 0.649 & \textbf{0.425} & 0.242 \\
& Fixed 
& 0.761 & 74.23 & 0.944 & 0.876 & 0.388 & 0.638 
& \textbf{0.873} & \underline{0.743} & 0.626 & 0.387 & 0.228 \\
& Loss 
& 0.758 & 73.42 & \underline{0.961} & 0.890 & 0.425 & 0.643 
& 0.847 & 0.726 & 0.632 & 0.394 & 0.257 \\
& Accu 
& 0.754 & 74.67 & 0.959 & 0.901 & 0.441 & 0.636 
& 0.864 & 0.731 & \underline{0.650} & 0.392 & 0.276 \\
& MSSR$_{\text{sch}}$ 
& \underline{0.772} & 75.83 & 0.952 & \underline{0.906} & 0.456 & \underline{0.654}
& 0.834 & 0.725 & 0.649 & 0.409 & \textbf{0.302} \\
& MSSR$_{\text{spl}}$ 
& 0.769 & \underline{76.10} & 0.957 & 0.903 & \underline{0.462} & 0.646
& 0.852 & 0.741 & 0.635 & 0.399 & 0.263 \\
& MSSR$_{\text{full}}$ 
& \textbf{0.785} & \textbf{77.16} & \textbf{0.969} & \textbf{0.911} & \textbf{0.487} & \textbf{0.665} 
& \underline{0.867} & \textbf{0.752} & \textbf{0.651} & \underline{0.412} & \underline{0.296} \\
\midrule

\multirow{7}{*}{Qwen2.5-7B}
& None 
& 0.622 & 70.59 & 0.926 & 0.878 & 0.385 & 0.540 
& \textbf{0.873} & 0.807 & 0.718 & 0.570 & 0.376 \\
& Fixed 
& 0.764 & 72.28 & 0.951 & 0.882 & 0.432 & 0.706 
& 0.841 & 0.780 & 0.701 & 0.606 & 0.383 \\
& Loss 
& 0.748 & 73.50 & 0.962 & 0.885 & 0.453 & 0.698 
& 0.842 & 0.776 & 0.725 & 0.611 & 0.380 \\
& Accu  
& 0.763 & \underline{73.89} & 0.955 & \underline{0.898} & 0.461 & 0.683 
& 0.849 & \textbf{0.822} & 0.751 & 0.594 & 0.372 \\
& MSSR$_{\text{sch}}$ 
& \underline{0.771} & 73.38 & 0.953 & 0.896 & 0.532 & 0.727 
& 0.825 & 0.798 & 0.763 & 0.596 & \underline{0.386} \\
& MSSR$_{\text{spl}}$ 
& 0.762 & 73.46 & \underline{0.966} & 0.893 & \underline{0.557} & \underline{0.733} 
& 0.855 & 0.786 & \underline{0.784} & \underline{0.623} & 0.375 \\
& MSSR$_{\text{full}}$ 
& \textbf{0.781} & \textbf{74.21} & \textbf{0.972} & \textbf{0.911} & \textbf{0.569} & \textbf{0.748} 
& \underline{0.864} & \underline{0.817} & \textbf{0.810} & \textbf{0.637} & \textbf{0.398} \\
\midrule

\multirow{7}{*}{Llama-3.1-8B}
& None 
& 0.781 & 76.03 & 0.920 & 0.837 & 0.598 & 0.532 
& 0.801 & 0.706 & 0.611 & 0.474 & 0.302 \\
& Fixed 
& 0.768 & 77.17 & 0.912 & 0.876 & 0.576 & 0.626 
& 0.783 & 0.715 & 0.671 & 0.518 & 0.329 \\
& Loss 
& \textbf{0.794} & 77.74 & 0.922 & 0.854 & 0.637 & 0.647
& \textbf{0.825} & 0.712 & 0.658 & 0.521 & 0.337 \\
& Accu 
& 0.771 & 78.53 & \underline{0.934} & 0.852 & 0.645 & \textbf{0.653} 
& 0.812 & 0.721 & 0.675 & 0.517 & \underline{0.352} \\
& MSSR$_{\text{sch}}$ 
& 0.756 & \underline{79.50} & 0.921 & 0.853 & \underline{0.688} & 0.635 
& 0.781 & 0.722 & \underline{0.694} & 0.507 & 0.345 \\
& MSSR$_{\text{spl}}$ 
& 0.775 & 78.26 & 0.933 & \underline{0.881} & 0.650 & 0.645 
& 0.814 & \underline{0.731} & 0.682 & \textbf{0.530} & 0.338 \\
& MSSR$_{\text{full}}$ 
& \underline{0.787} & \textbf{79.95} & \textbf{0.945} & \textbf{0.889} & \textbf{0.692} & \underline{0.651} 
& \textbf{0.825} & \textbf{0.746} & \textbf{0.705} & \underline{0.524} & \textbf{0.363} \\
\bottomrule
\end{tabular}
}
\end{table*}

% Insert Figure/Table: retention vs. buffer size
\begin{table}[h]
    \caption{Effect of buffer size on retention performance under the Ebbinghaus scheduler.}
    \label{tab:buffer_size}
    \centering
    \small
    \begin{tabular}{c|c}
        \toprule
        Buffer Size & Average Retention (\%) \\
        \midrule
        512   & 62.1 \\
        1024  & 71.4 \\
        2048  & 78.9 \\
        4096  & 80.3 \\
        \bottomrule
    \end{tabular}
\end{table}

\begin{table}[h]
\caption{
\textbf{Replay interval comparison.} 
Ebbinghaus-inspired spacing achieves the highest accuracy and lowest forgetting. 
(Uniform, Heuristic, and Cognitive correspond to the three row types.)
}
\label{tab:scheduler_variants}
\centering
\small
\setlength{\tabcolsep}{3.5pt}
\begin{tabular}{lcc}
\toprule
\textbf{Interval Sequence} & \textbf{Acc. (\%)} & \textbf{Forget ↓} \\
\midrule
Fixed ($\Delta t_r{=}3$)        & 84.6 & 0.0 \\
Geometric $\{1,3,7,14,30\}$     & 87.2 & 2.6 \\
Ebbinghaus $\{1,2,4,7,15\}$     & \textbf{89.1} & \textbf{4.5} \\
\bottomrule
\end{tabular}
\end{table}

% ===== Table 1: Replay Ratio =====
\begin{table}[h]
\centering
\caption{Sensitivity to initial replay ratio $\lambda_0$ (Qwen2.5-7B, 11-task sequence).}
\label{tab:ratio_ablation}
\resizebox{0.5\textwidth}{!}{%
\begin{tabular}{lccccccc}
\toprule
Strategy / $\lambda_0$ & fixed & loss & eval & 0.05 & 0.10 & 0.20 & 0.30 \\
\midrule
GSM8K Avg Acc & 0.706 & 0.698 & 0.727 & 0.716 & 0.725 & 0.748 & 0.731 \\
BoolQ Avg Acc & 0.876 & 0.854 & 0.865 & 0.874 & 0.892 & 0.889 & 0.883 \\
\bottomrule
\end{tabular}%
}
\end{table}

% ===== Table 2: Buffer Size =====
\begin{table}[h]
\centering
\caption{Effect of buffer size on average retention (\%).}
\label{tab:buffer_size}
\resizebox{0.5\textwidth}{!}{%
\begin{tabular}{c|cccc}
\toprule
Buffer Size (samples) & 512 & 1024 & 2048 & 4096 \\
\midrule
Average Retention (\%) & 68.7 & 73.5 & 75.9 & 76.8 \\
\bottomrule
\end{tabular}
}
\end{table}

% ===== Table 3: Computational Overhead =====
\begin{table}[h]
\centering
\caption{Training and memory overhead of MSSR (normalized to fixed replay).}
\label{tab:overhead}
\resizebox{0.5\textwidth}{!}{
\begin{tabular}{l|cccc}
\toprule
Method & Wall-clock Time & Peak Memory & Throughput (samples/s) & Per-step Latency  \\
\midrule
Fixed Replay & 1.00 & 1.00 & 1.00 & 1.00  \\
MSSR$_{\text{spl}}$ & 1.03 & 1.04 & 0.99 & 1.02  \\
MSSR$_{\text{sch}}$ & 1.04 & 1.05 & 0.98 & 1.02  \\
MSSR$_{\text{full}}$ & 1.05 & 1.06 & 0.98 & 1.05  \\
\bottomrule
\end{tabular}%
}
\end{table}

\begin{table}[h]
\centering
\caption{Effect of interval scheduling on average accuracy (\%) and forgetting drop (\%).}
\label{tab:scheduler_variants}
\resizebox{0.5\textwidth}{!}{
\begin{tabular}{c|c|cc}
\toprule
Interval Pattern & Interval Sequence & Avg Acc (\%) & Forget Drop (\%) \\
\midrule
Fixed & {3,3,3,...} & 76.1 & 4.8 \\
Geometric & {1,3,7,15,...} & 77.8 & 3.1 \\
\tool & {1,2,4,7,15,...} & 78.5 & 2.4 \\
\bottomrule
\end{tabular}
}
\end{table}

%% file: sections/5-related_work.tex
\section{Related Work}
\subsection{Catastrophic Forgetting in LLM Fine-Tuning}

When Large Language Models (LLMs) are fine-tuned sequentially across tasks, they often suffer from catastrophic forgetting, where performance on prior tasks degrades as new knowledge is acquired~\citep{van2024continual}. 
In the continual learning (CL) literature, mitigation strategies are commonly categorized into parameter regularization (e.g., \citet{kirkpatrick2017overcoming,zenke2017continual}), knowledge distillation, and architectural isolation. 
However, scaling these approaches to LLMs remains challenging: regularization requires reliable parameter importance estimation across billions of weights, distillation introduces additional training and storage overhead~\citep{shi2024continual,xu2024survey}, and architectural methods complicate deployment and adaptation~\citep{rusu2016progressive,shi2024continual}. 
As a result, experience replay has emerged as a particularly practical solution, as it operates at the data level and integrates naturally with parameter-efficient fine-tuning methods such as LoRA~\citep{rolnick2019experience,aljundi2019online,chaudhry2019tiny,hu2022lora,wang2025parameter}.

\subsection{Replay-Based Strategies for LLMs}
Replay alleviates forgetting by reintroducing samples from past tasks during training. 
Early approaches adopt fixed interleaving, inserting replay batches at uniform intervals, which is computationally efficient but ignores forgetting dynamics~\citep{rolnick2019experience}. 
Subsequent work explores priority sampling~\citep{aljundi2019online,schaul2015prioritized,aljundi2019gradient} and dynamic scheduling based on validation accuracy or training loss, improving retention at the cost of frequent evaluations and increased computation. 
Recent studies further investigate \textit{logit-level replay}~\citep{buzzega2020dark,buzzega2021rethinking} and \textit{synthetic replay}~\citep{shin2017continual,huang2024mitigating}, where pseudo-samples are generated when original data are unavailable. 
Memory-augmented approaches extend replay with explicit long-term storage and consolidation mechanisms~\citep{rebuffi2017icarl,zhong2024memorybank,shan2025cognitive}. 
Despite these advances, existing methods still face trade-offs between effectiveness and efficiency, motivating the search for more principled and lightweight replay scheduling strategies~\citep{li2025cmt,gutierrez2025rag}.

\subsection{Cognitive-Inspired Replay Scheduling}

Cognitive science suggests that forgetting follows a nonlinear trajectory: the Ebbinghaus forgetting curve shows that memory decays exponentially and can be reinforced through spaced repetition at expanding intervals~\citep{ebbinghaus2013image,wixted2004psychology,cepeda2006distributed,pashler2007enhancing}. 
Although spaced repetition is widely used in education and cognitive modeling, its application to replay scheduling in LLM continual learning remains limited~\citep{shi2024continual,zheng2025towards}. 
Some recent studies explore learned schedulers via reinforcement learning or search-based methods~\citep{klasson2022learn}, but these approaches often incur substantial computational cost and lack interpretability. 
Grounding replay timing in the Ebbinghaus principle offers a lightweight and theoretically motivated alternative that better aligns with the temporal dynamics of forgetting~\citep{mozer2008optimal,settles2016trainable,cepeda2008spacing}, motivating our approach to cognitively inspired replay scheduling for LLM fine-tuning.

%% file: sections/6-conclusion.tex
\section{Conclusion}
We proposed \tool, a memory-aware replay framework for continual fine-tuning of LLMs that combines sample-level retention modeling with adaptive replay scheduling in a parameter-efficient LoRA pipeline. 

Across both 3-task and extended 11-task continual learning settings, \tool\ delivers stable gains over fixed, loss-based, and accuracy-based replay on multiple backbones (Qwen2.5-7B, LLaMA-3.1-8B, Gemma2-9B), with the largest improvements on long-context and reasoning benchmarks where early-task forgetting is most severe. Ablations show that \tool\ remains robust to replay ratio, buffer size, and scheduling choices, and achieves reliable long-term retention with minimal computational and memory overhead.

Overall, \tool\ provides a practical and scalable approach to long-horizon continual fine-tuning, balancing retention, efficiency, and interpretability.

%% file: sections/impact_statement.tex
\section*{Impact Statement}
This paper presents work whose goal is to advance the field
of Machine Learning. There are many potential societal
consequences of our work, none which we feel must be
specifically highlighted here.

%% file: sections/Acknowledgements.tex
\section*{Acknowledgements}
We acknowledge USTC for providing computational resources, and thank Dr. Lin Yang and Hanzhu Chen for helpful discussions.

%% file: sections/appendix.tex
\appendix

\section{Derivations for Sample-Level Memory Dynamics}
\label{app:sample_level_derivations}

This appendix provides derivations and technical clarifications for the
sample-level memory model introduced in Section~\ref{method: sample_modeling}
(see Eqs.~\ref{eq:retention_update}--\ref{eq:epoch_update} in the main text).
We state the modeling assumptions behind the hazard-based retention dynamics,
its continuous-time survival form, the epoch-level discretization,
and the consolidation rule applied at review events.

\subsection{Notation and Review-Event Definition}
\label{app:notation}

We index training by discrete steps $t\in\mathbb{N}$. For each sample $i$, we maintain:
\begin{itemize}
    \item \textbf{Memory strength} $m_{i,t}\in(0,1]$, interpreted as the retained strength (or survival probability) of sample $i$ at step $t$;
    \item \textbf{Stability} $S_{i,t}>0$, controlling resistance to forgetting (larger $S_{i,t}$ implies slower decay);
    \item \textbf{Instantaneous hazard} $h_{i,t}\ge 0$, the per-step decay rate used in the multiplicative retention update;
    \item \textbf{Normalized loss} $\widehat{\ell}_{i,t}\in[0,1]$, a denoised, scale-stabilized per-sample loss signal;
    \item A monotone mapping $\phi:[0,1]\rightarrow[0,1]$ applied to $\widehat{\ell}_{i,t}$ (e.g., identity or a calibrated sigmoid).
\end{itemize}
We also use nonnegative scalars $\alpha_i$ (baseline drift) and $\gamma_d$ (difficulty sensitivity).

\paragraph{Review times and inter-review interval.}
We denote by $\mathcal{R}_i\subset\mathbb{N}$ the \emph{set of review times} of sample $i$, i.e., the set of training steps at which sample $i$ is explicitly revisited (replayed) by the sampler/scheduler. Formally,
\[
t\in\mathcal{R}_i \quad \Longleftrightarrow \quad \text{sample $i$ is selected and used for training at step $t$ as a replay exposure.}
\]
For any review step $t\in\mathcal{R}_i$, we use $t^{+}$ to denote the \emph{post-review} state immediately after the review update is applied. Let
\[
t_i^{\star}(t)\;=\;\max\{\tau\in\mathcal{R}_i:\,\tau<t\}
\]
be the most recent review time strictly before $t$ (if it exists). The elapsed steps since the previous review are
\[
\Delta t_i(t)\;=\;t-t_i^{\star}(t),
\]
which is the discrete inter-review interval used in the consolidation rule. When the dependence on $t$ is clear, we abbreviate $\Delta t_i(t)$ as $\Delta t_i$.

\paragraph{Core per-step retention update.}
Between review events, the memory strength decays multiplicatively:
\begin{equation}
m_{i,t+1}=m_{i,t}\exp(-h_{i,t}),
\qquad
h_{i,t}=\frac{\alpha_i+\gamma_d\,\phi(\widehat{\ell}_{i,t})}{S_{i,t}},
\label{eq:retention_update_app}
\end{equation}
which guarantees $m_{i,t}\in(0,1]$ for all $t$ given $m_{i,0}\in(0,1]$ and $h_{i,t}\ge 0$.

\subsection{Loss Denoising and Normalization}
\label{app:loss_norm}

Let $\ell_{i,t}$ denote the raw per-sample loss (e.g., token-averaged NLL). We compute a denoised loss using an exponential moving average (EMA):
\begin{equation}
\widetilde{\ell}_{i,t}
= \beta_{\mathrm{ema}}\,\widetilde{\ell}_{i,t-1} + (1-\beta_{\mathrm{ema}})\,\ell_{i,t},
\quad \beta_{\mathrm{ema}}\in(0,1).
\label{eq:ema_app}
\end{equation}
We then apply robust quantile normalization to obtain $\widehat{\ell}_{i,t}\in[0,1]$:
\begin{equation}
\widehat{\ell}_{i,t}
= \mathrm{clip}\!\left(
\frac{\widetilde{\ell}_{i,t}-Q_{q_l}}{Q_{q_u}-Q_{q_l}},
\,0,\,1\right),
\label{eq:quantile_app}
\end{equation}
where $Q_{q}$ denotes a running quantile of $\widetilde{\ell}_{i,t}$ and $(q_l,q_u)$ are fixed quantile levels.

\subsection{Closed-Form Between Reviews and Survival Form}
\label{app:survival}

Unrolling the multiplicative decay between two reviews gives a closed form.

\begin{lemma}[Closed form between reviews]
\label{lem:closed_form}
Fix a review time $t_i^{\star}$ and assume no review occurs on $[t_i^{\star},t)$.
Then
\begin{equation}
m_{i,t}
= m_{i,t_i^{\star}}\,
\exp\!\Big(-\sum_{\tau=t_i^{\star}}^{t-1} h_{i,\tau}\Big).
\label{eq:closed_form_app}
\end{equation}
\end{lemma}
\begin{proof}
For $\tau=t_i^{\star},\dots,t-1$, iterating the recurrence yields
\[
m_{i,t}
= m_{i,t_i^{\star}}\prod_{\tau=t_i^{\star}}^{t-1}\exp(-h_{i,\tau}).
\]
Using $\prod_{\tau}\exp(-h_{i,\tau})=\exp(-\sum_{\tau}h_{i,\tau})$ completes the proof.
\end{proof}

\paragraph{Stationary (piecewise-constant) hazard over an inter-review interval.}
On an interval without review events $[t_i^{\star},t)$, the stability is unchanged by design
(i.e., $S_{i,\tau}=S_{i,t_i^{\star}}$ for all $\tau\in[t_i^{\star},t)$), since $S_{i,t}$ is updated only at $t\in\mathcal{R}_i$.
Moreover, we approximate the difficulty signal by an interval-average (or epoch-average) loss:
\[
\bar{\ell}_{i,[t_i^{\star},t)} \;=\; \frac{1}{t-t_i^{\star}}\sum_{\tau=t_i^{\star}}^{t-1}\widehat{\ell}_{i,\tau},
\qquad
\bar{\phi}_{i,[t_i^{\star},t)} \;=\; \phi\!\big(\bar{\ell}_{i,[t_i^{\star},t)}\big).
\]
Under these approximations, the hazard becomes approximately constant on $[t_i^{\star},t)$:
\begin{equation}
h_{i,\tau}\;\approx\;\bar{h}_{i,[t_i^{\star},t)}
\;=\;\frac{\alpha_i+\gamma_d\,\bar{\phi}_{i,[t_i^{\star},t)}}{S_{i,t_i^{\star}}},
\qquad \tau\in[t_i^{\star},t).
\label{eq:approx_const_hazard}
\end{equation}
Substituting into Eq.~\eqref{eq:closed_form_app} yields the exponential retention form
\begin{equation}
m_{i,t}\;\approx\; m_{i,t_i^{\star}}\,\exp\!\big(-\bar{h}_{i,[t_i^{\star},t)}\,\Delta t_i\big),
\qquad \Delta t_i=t-t_i^{\star},
\label{eq:stationary_app_revised}
\end{equation}
which recovers an Ebbinghaus-style decay with effective rate $\bar{h}_{i,[t_i^{\star},t)}$.

\subsection{Epoch-Level Approximation and Error Control}
\label{app:epoch}

To reduce the overhead of updating $m_{i,t}$ at every training step, we approximate the hazard dynamics by an epoch-wise (piecewise-constant) hazard. Let $\{T_0,\dots,T_E\}$ be epoch boundaries and $\Delta t_e=T_e-T_{e-1}$. The main text adopts
\begin{equation}
h_{i,T_e}=\frac{\alpha_i+\gamma_d\,\phi(\widehat{\ell}_{i,T_e})}{S_{i,T_e}},
\qquad
m_{i,T_e}=m_{i,T_{e-1}}\exp(-h_{i,T_e}\,\Delta t_e),
\tag{\ref{eq:epoch_update} revisited}
\end{equation}
i.e., a right-endpoint Riemann discretization (using epoch-end statistics for $\widehat{\ell}$ and $S$).

\begin{proposition}[Riemann approximation error]
\label{prop:riemann}
Let $\tilde m_i(t)=\tilde m_i(T_{e-1})\exp\!\big(-\int_{T_{e-1}}^{t} h_i(\tau)\,d\tau\big)$ be the continuous-time solution on $[T_{e-1},T_e]$.
If $h_i(\tau)$ is $L$-Lipschitz on $[T_{e-1},T_e]$, then
\begin{equation}
\big|\log m_{i,T_e}-\log \tilde m_i(T_e)\big|
\le \tfrac{L}{2}\,(\Delta t_e)^2.
\label{eq:riemann_error}
\end{equation}
Hence the epoch-level approximation converges to the continuous-time solution as $\max_e \Delta t_e\to 0$.
\end{proposition}
\begin{proof}
Let $I=\int_{T_{e-1}}^{T_e} h_i(\tau)\,d\tau$ and $Q=h_i(T_e)\,\Delta t_e$.
By $L$-Lipschitz continuity, $|h_i(\tau)-h_i(T_e)|\le L|\tau-T_e|$ for all $\tau$ in the interval, so
$|I-Q|\le\int_{T_{e-1}}^{T_e} L|\tau-T_e|\,d\tau = \tfrac{L}{2}(\Delta t_e)^2$.
Since $\log \tilde m_i(T_e)-\log \tilde m_i(T_{e-1})=-I$ and $\log m_{i,T_e}-\log m_{i,T_{e-1}}=-Q$, the claim follows.
\end{proof}

\paragraph{Lazy update.}
In practice, we store $(t_i^\star, m_{i,t_i^\star}, S_{i,t_i^\star})$ and update $m_{i,t}$ only when queried.
A convenient implementation uses
$m_{i,t}=m_{i,t_i^\star}\exp(-\widehat h_{i,e}\,\Delta t_i)$,
where $\widehat h_{i,e}$ is the most recent epoch-level hazard estimate (e.g., $h_{i,T_e}$) and $\Delta t_i=t-t_i^\star$.
This avoids per-step updates while preserving the same survival form.

\subsection{Generalized Consolidation at Review}
\label{app:gcr}

At review times $t\in\mathcal{R}_i$, we reset retention and increase stability according to
\begin{equation}
\begin{aligned}
 m_{i,t^{+}} &= 1, \\
 S_{i,t^{+}} &= S_{i,t}
 + \eta_s\,(S_{\max}-S_{i,t})^{\beta}
 \exp\!\big(-\rho\,\Delta t_i(t)\big)
 (1-m_{i,t})^{\gamma_s}
 + \epsilon_t,
\end{aligned}
\tag{\ref{eq:reset_spacing} revisited}
\end{equation}
where $\epsilon_t\sim\mathcal{N}(0,\sigma_s^2)$ and $\Delta t_i(t)=t-t_i^{\star}$.

\paragraph{Monotonicity and boundedness.}
Assume $\eta_s\!\ge\!0$, $\beta\!\in\!(0,1]$, $\rho\!\ge\!0$, $\gamma_s\!\ge\!0$, and $\mathbb{E}[\epsilon_t]=0$.
Then $\mathbb{E}[S_{i,t^{+}}\mid S_{i,t}]\ge S_{i,t}$ and the expected increment vanishes as
$S_{i,t}\to S_{\max}$ because $(S_{\max}-S_{i,t})^{\beta}\to 0$.
In implementations we additionally clip $S_{i,t}\in[S_{\min},S_{\max}]$ to enforce boundedness.

\paragraph{Spacing implication (time-to-threshold).}
Between reviews, with average hazard $\bar h_i\approx(\alpha_i+\gamma_d\phi(\widehat{\ell}_i))/S_{i,t}$,
the time $\tau^\star$ to hit a retention threshold $\theta\in(0,1)$ satisfies
\begin{equation}
\tau^\star
= \frac{S_{i,t}}{\alpha_i+\gamma_d\phi(\widehat{\ell}_i)}\,
\log\!\frac{1}{\theta},
\label{eq:timetothreshold_app}
\end{equation}
which increases with stability $S_{i,t}$. Thus repeated consolidation lengthens the optimal inter-review interval.

\subsection{Choices of the Monotone Mapping $\phi$}
\label{app:phi}

The mapping $\phi:[0,1]\!\to\![0,1]$ converts the normalized loss $\widehat{\ell}_{i,t}$ into a monotone difficulty signal that modulates the hazard.
In all experiments, we instantiate $\phi$ as a calibrated sigmoid (Table~\ref{tab:memory_hparams}), while other monotone choices yield similar behavior.

\paragraph{Identity.}
$\phi(x)=x$ (linear sensitivity). This is simple but can be more sensitive to noisy loss spikes.

\paragraph{Calibrated sigmoid (default).}
\begin{equation}
\phi(x)=\frac{1}{1+\exp[-k(x-c)]},\qquad k>0,\; c\in(0,1).
\label{eq:sigmoid_phi}
\end{equation}
which saturates for extreme losses and concentrates sensitivity around the difficulty band near $c$ (with slope controlled by $k$).

\paragraph{Power/log mappings.}
$\phi(x)=x^{p}$ ($p>0$) or $\phi(x)=\log(1+\kappa x)/\log(1+\kappa)$ ($\kappa>0$),
allowing sub-/super-linear emphasis on harder examples.

All choices preserve monotonicity. Since $S_{i,t}>0$ and $\gamma_d\ge 0$, we have
$\partial h_{i,t}/\partial \widehat{\ell}_{i,t} = (\gamma_d/S_{i,t})\,\phi'(\widehat{\ell}_{i,t}) \ge 0$
(where differentiable), hence larger loss implies a larger hazard at fixed stability.

\subsection{Hyperparameter Settings and Practical Defaults}
\label{app:hyperparams}

Table~\ref{tab:memory_hparams} summarizes the hyperparameters used in the sample-level memory model.
Unless otherwise stated, all hyperparameters are shared across samples.

\begin{table}[t]
\centering
\small
\begin{tabular}{lcccc}
\toprule
\textbf{Parameter} & \textbf{Default} & \textbf{Candidate Values} & \textbf{Role} & \textbf{Where used} \\
\midrule
$\alpha$ (shared) & $0.01$ & $\{0,\,0.005,\,0.01,\,0.02,\,0.05\}$ & baseline decay
& Eq.~\eqref{eq:retention_update}, Eq.~\eqref{eq:epoch_update} \\
$\gamma_d$ & $0.20$ & $\{0.05,\,0.10,\,0.20,\,0.30,\,0.50\}$ & loss sensitivity
& Eq.~\eqref{eq:retention_update}, Eq.~\eqref{eq:epoch_update} \\
$\phi(\cdot)$ & sigmoid & --- & monotone mapping
& Eq.~\eqref{eq:retention_update}, Eq.~\eqref{eq:epoch_update} (via $\phi(\widehat{\ell})$) \\
$k$ (sigmoid) & $10$ & $[5,20]$ & slope / hardness band width
& Eq.~\eqref{eq:sigmoid_phi} \\
$c$ (sigmoid) & $0.5$ & $[0.4,0.6]$ & center of difficulty band
& Eq.~\eqref{eq:sigmoid_phi} \\
EMA $\beta_{\mathrm{ema}}$ & $0.95$ & $\{0.90,\,0.95,\,0.97,\,0.99\}$ & loss denoising
& Eq.~\eqref{eq:ema_app} \\
Quantiles $(q_l,q_u)$ & $(0.05,0.95)$ & fixed & robust normalization
& Eq.~\eqref{eq:quantile_app} \\
\midrule
$\eta_s$ & $0.05$ & $\{0.01,\,0.02,\,0.05,\,0.10,\,0.20\}$ & consolidation step size
& Eq.~\eqref{eq:reset_spacing} \\
$\beta_s$ & $0.5$ & $\{0.5,\,0.75,\,1.0\}$ & saturation exponent
& Eq.~\eqref{eq:reset_spacing} \\
$\rho$ & $0.01$ & $\{0,\,0.005,\,0.01,\,0.02,\,0.05\}$ & spacing sensitivity
& Eq.~\eqref{eq:reset_spacing} (via $e^{-\rho\Delta t_i}$) \\
$\gamma_s$ & $1.0$ & $\{0.5,\,1.0,\,2.0\}$ & error-driven reinforcement
& Eq.~\eqref{eq:reset_spacing} (via $(1-m_{i,t})^{\gamma_s}$) \\
$S_{\min},S_{\max}$ & $(1,10)$ & $S_{\max}\in\{5,10,20\}$ & stability clipping
& Eq.~\eqref{eq:reset_spacing}, text in App.~\ref{app:gcr} \\
$\sigma_s$ & $0$ & fixed & noise std
& Eq.~\eqref{eq:reset_spacing} (via $\epsilon_t\sim\mathcal{N}(0,\sigma_s^2)$) \\
\bottomrule
\end{tabular}
\caption{Hyperparameters for the sample-level memory dynamics (Appendix~\ref{app:sample_level_derivations}).}
\label{tab:memory_hparams}
\end{table}

\section{Derivations for Dataset-Level Replay Scheduling}
\label{app:replay_derivation}

This appendix provides derivations supporting the dataset-level replay scheduling equations
(Eqs.~\ref{eq:interval_expand}--\ref{eq:prob_weight}) in the main text.
Our goal is to connect the micro-level sample dynamics (Appendix~\ref{app:sample_level_derivations})
to macro-level replay decisions---\emph{when to replay}, \emph{how much to replay}, and \emph{which samples to replay}.

\subsection{Setup, Indexing, and Mean-Field Aggregation}
\label{app:replay_setup}

We index dataset-level replay cycles by $k=1,2,\dots$. Let $t_k$ denote the training step at which the $k$-th
dataset-level replay event is executed, and define the inter-replay interval
\[
\Delta t_r^{(k)} \;=\; t_k - t_{k-1}.
\]
At cycle $k$, we also maintain a replay ratio $\lambda_{t_k}\in[0,1]$ that controls the fraction of replayed samples in a batch,
and a per-sample replay probability $p_i^{(t_k)}$ over the replay buffer $B_{\mathrm{replay}}$.

To relate dataset-level scheduling to sample-level states, we adopt a mean-field (population-averaged) view:
we summarize the sample-level hazard and stability by aggregated statistics at cycle $k$,
\[
\bar{h}_k \;\approx\; \mathbb{E}_{i\sim B_{\mathrm{replay}}}\!\left[h_{i,t_k}\right],
\qquad
S_k \;\approx\; \mathbb{E}_{i\sim B_{\mathrm{replay}}}\!\left[S_{i,t_k}\right],
\qquad
\overline{\phi(\widehat{\ell})}\;\approx\;\mathbb{E}_{i\sim B_{\mathrm{replay}}}\!\left[\phi(\widehat{\ell}_{i,t_k})\right].
\]
This aggregation is used only to motivate a tractable global schedule; the actual algorithm operates on per-sample states.

\subsection{Derivation of Replay Interval Expansion}
\label{app:interval_expansion}

We derive why replay intervals naturally expand over cycles as stability accumulates.
Starting from the sample-level hazard form (Eq.~\ref{eq:retention_update}), we adopt a mean-field approximation at replay cycle $k$:
\begin{equation}
\bar{h}_k \;\approx\;
\mathbb{E}_{i\sim B_{\mathrm{replay}}}\!\left[h_{i,t_k}\right]
\;\approx\;
\frac{\bar{\alpha} + \bar{\gamma}_d\,\overline{\phi(\widehat{\ell})}_k}{S_k},
\label{eq:avg_hazard_k}
\end{equation}
where $\bar{\alpha}\approx\mathbb{E}[\alpha_i]$, $\bar{\gamma}_d$ is the shared loss-sensitivity coefficient,
$S_k\approx\mathbb{E}[S_{i,t_k}]$, and $\overline{\phi(\widehat{\ell})}_k\approx\mathbb{E}[\phi(\widehat{\ell}_{i,t_k})]$.
Within a single cycle, we treat $\overline{\phi(\widehat{\ell})}_k$ as slowly varying compared to the evolution of $S_k$,
so that the dominant driver of $\bar{h}_k$ is $1/S_k$.

\paragraph{Threshold-triggered replay and $\tau_k^\star\propto S_k$.}
Assuming exponential retention decay between two dataset-level replay events,
the population-average retention follows $\bar{m}(\tau)=\exp(-\bar{h}_k\,\tau)$.
If the next replay is triggered when $\bar{m}$ drops below a fixed threshold $\theta\in(0,1)$, then
\begin{equation}
\tau_k^{\star}
\;=\;
\frac{1}{\bar{h}_k}\log\!\frac{1}{\theta}
\;=\;
\frac{S_k}{\bar{\alpha} + \bar{\gamma}_d\,\overline{\phi(\widehat{\ell})}_k}
\log\!\frac{1}{\theta},
\qquad
\text{i.e.,}\quad \tau_k^\star \propto S_k.
\label{eq:tau_star}
\end{equation}
In practice, batching and discrete steps introduce mild deviations; we therefore use this threshold-derived optimal spacing
as a proxy for the realized schedule, i.e., $\Delta t_r^{(k)}\approx \tau_k^\star$.

\paragraph{From stability growth to interval expansion.}
From the sample-level consolidation rule (Eq.~\ref{eq:reset_spacing}), stability increases after replay.
Aggregating the per-sample updates yields the following schematic population-level recursion:
\begin{equation}
S_{k+1} \;=\; S_k \;+\; \Delta S_k,
\qquad
\Delta S_k
\;=\;
\eta_s (S_{\max} - S_k)^{\beta_s}
\exp\!\big(-\rho\,\Delta t_r^{(k)}\big)
(1 - \bar{m}_k)^{\gamma_s},
\label{eq:deltaS_appendix}
\end{equation}
where $\bar{m}_k$ denotes the average retention just before the $k$-th replay.
Since $\Delta t_r^{(k)}\approx \tau_k^\star \propto S_k$, the ratio of consecutive replay intervals satisfies
\begin{equation}
\frac{\Delta t_r^{(k+1)}}{\Delta t_r^{(k)}}
\;\approx\;
\frac{S_{k+1}}{S_k}
\;=\;
1+\frac{\Delta S_k}{S_k},
\label{eq:ratio_intervals}
\end{equation}
highlighting that interval expansion is governed by the relative stability gain $\Delta S_k/S_k$.

\paragraph{Exponential envelope for diminishing relative gains.}
The relative gain $\Delta S_k/S_k$ decreases with $k$ because:
(i) $(S_{\max}-S_k)^{\beta_s}$ shrinks as $S_k$ approaches $S_{\max}$ (saturating consolidation),
and (ii) $\Delta t_r^{(k)}$ increases, which further reduces $\exp(-\rho\,\Delta t_r^{(k)})$ (spacing attenuation).
To obtain a simple closed-form schedule, we \emph{parameterize} this diminishing trend using an exponential envelope:
\begin{equation}
\frac{\Delta S_k}{S_k} \;\approx\; \eta_p \exp(-\rho_p k),
\label{eq:spacing_gain_decay}
\end{equation}
where $\eta_p$ and $\rho_p$ are population-level hyperparameters summarizing the net effects of
$(\eta_s,\beta_s,\rho,\gamma_s)$ and the evolution of $\bar{m}_k$.
Substituting Eq.~\eqref{eq:spacing_gain_decay} into Eq.~\eqref{eq:ratio_intervals} yields the \textbf{interval expansion rule}:
\begin{equation}
\Delta t_r^{(k+1)} \;=\; \Delta t_r^{(k)}\left(1 + \eta_p e^{-\rho_p k}\right),
\label{eq:interval_expand_appendix}
\end{equation}
which corresponds to Eq.~\ref{eq:interval_expand} in the main text.
Since $1 + \eta_p e^{-\rho_p k} > 1$ for all $k$, replay is dense early and becomes progressively sparser as $k$ increases.

\subsection{Derivation of the Dynamic Replay Ratio}
\label{app:ratio_derivation}

Let $\lambda_t$ denote the proportion of replayed samples mixed into the batch at time $t$.
We view $\lambda_t$ as a resource-allocation variable trading \emph{marginal benefit} against \emph{marginal cost}.
Let the instantaneous replay benefit be $A_t g(\lambda)$, where $A_t\ge 0$ is the current \emph{rehearsal utility}
and $g(\lambda)$ is concave (diminishing returns), e.g., $g(\lambda)=1-e^{-b\lambda}$ with $b>0$.
We consider the proxy objective
\begin{equation}
\max_{\lambda \in [0,1]} \;\; A_t g(\lambda) - \mu \lambda,
\label{eq:ratio_optim}
\end{equation}
where $\mu>0$ weights compute cost. The stationarity condition gives
$A_t b e^{-b\lambda_t^\star}=\mu$, hence
\begin{equation}
\lambda_t^\star = \frac{1}{b}\log\!\frac{A_t b}{\mu}.
\label{eq:lambda_star}
\end{equation}

As fine-tuning proceeds, $A_t$ decreases due to stability accumulation and diminishing replay gains, so $\lambda_t^\star$ should decrease.
We adopt an exponential schedule as a simple first-order approximation to this decreasing trajectory:
\begin{equation}
\lambda_{t_k} = \lambda_{\min} +
(\lambda_0 - \lambda_{\min}) e^{-\beta_r t_k},
\label{eq:ratio_decay_appendix}
\end{equation}
corresponding to Eq.~\ref{eq:dynamic_ratio} in the main text.
This maintains higher replay intensity early and gradually approaches a small background replay rate $\lambda_{\min}$.

\subsection{Coupling Replay Probability with Memory Strength}
\label{app:prob_derivation}

We finally derive the per-sample replay probability $p_i^{(t_k)}$ used to select samples from the replay buffer $B_{\mathrm{replay}}$.
A natural (but noisy) proxy for replay utility is the expected consolidation gain,
which increases with forgettability and error-driven plasticity:
\begin{equation}
G_i(t) \;\propto\;
\frac{\alpha_i + \gamma_d\,\phi(\widehat{\ell}_{i,t})}{S_{i,t}}
(1 - m_{i,t})^{\gamma_s}.
\label{eq:gain_proxy}
\end{equation}
Directly sampling from \eqref{eq:gain_proxy} would require maintaining multiple high-variance signals.
Instead, we use the retention state $m_{i,t}$---already maintained by the sample-level dynamics and amenable to lazy updates---
as a low-variance surrogate. This yields a normalized power-law prioritization:
\begin{equation}
p_i^{(t_k)} \;=\;
\frac{m_{i,t_k}^{-\zeta}}
{\sum_{j\in B_{\mathrm{replay}}} m_{j,t_k}^{-\zeta}},
\qquad \zeta>0,
\label{eq:prob_weight_appendix}
\end{equation}
which matches Eq.~\ref{eq:prob_weight} in the main text.
Here $\zeta$ controls the strength of prioritization: $\zeta=0$ reduces to uniform sampling, while larger $\zeta$
biases selection toward low-retention (high-risk) samples.
Equivalently, $p_i^{(t_k)}\propto \exp(-\zeta\log m_{i,t_k})$.

\subsection{Summary of Hierarchical Dynamics}
\label{app:replay_summary}

Combining Eqs.~\eqref{eq:interval_expand_appendix}, \eqref{eq:ratio_decay_appendix}, and \eqref{eq:prob_weight_appendix},
the dataset-level replay schedule can be summarized as:
\begin{itemize}
    \item \textbf{Temporal spacing:} replay intervals $\Delta t_r^{(k)}$ expand as stability accumulates;
    \item \textbf{Replay intensity:} replay ratio $\lambda_{t}$ decays toward a stable baseline;
    \item \textbf{Replay prioritization:} sampling probability $p_i$ prioritizes weaker-retention samples via $m_{i,t}$.
\end{itemize}
Together, these mechanisms form a hierarchical replay policy that aligns micro-level memory decay
with macro-level temporal scheduling and replay resource allocation.

\subsection{Hyperparameter Settings for Dataset-Level Scheduling}
\label{app:replay_hparams}

For completeness, Table~\ref{tab:replay_hparams} summarizes the dataset-level scheduling hyperparameters.

\begin{table}[t]
\centering
\small
\begin{tabular}{lcccc}
\toprule
\textbf{Parameter} & \textbf{Default} & \textbf{Candidate Values} & \textbf{Role} & \textbf{Where used} \\
\midrule
$\Delta t_r^{(1)}$ & $100$ & $\{50,\,100,\,200,\,500\}$ & initial replay interval
& Eq.~\eqref{eq:interval_expand_appendix} \\
$\theta$ & $0.5$ & $\{0.3,\,0.5,\,0.7\}$ & replay trigger threshold
& Eq.~\eqref{eq:tau_star} \\
$\eta_p$ & $0.5$ & $\{0.1,\,0.3,\,0.5,\,0.7,\,1.0\}$ & early expansion gain
& Eq.~\eqref{eq:spacing_gain_decay}, Eq.~\eqref{eq:interval_expand_appendix} \\
$\rho_p$ & $0.05$ & $\{0.01,\,0.03,\,0.05,\,0.10\}$ & decay rate of expansion gain
& Eq.~\eqref{eq:spacing_gain_decay}, Eq.~\eqref{eq:interval_expand_appendix} \\
\midrule
$\lambda_0$ & $0.3$ & $\{0.1,\,0.3,\,0.5\}$ & initial replay ratio
& Eq.~\eqref{eq:ratio_decay_appendix} \\
$\lambda_{\min}$ & $0.05$ & $\{0,\,0.01,\,0.05,\,0.10\}$ & minimal replay ratio
& Eq.~\eqref{eq:ratio_decay_appendix} \\
$\beta_r$ & $1\mathrm{e}{-5}$ & $\{5\mathrm{e}{-6},\,1\mathrm{e}{-5},\,2\mathrm{e}{-5}\}$ & ratio decay rate
& Eq.~\eqref{eq:ratio_decay_appendix} \\
\midrule
$\zeta$ & $1.0$ & $\{0,\,0.5,\,1.0,\,2.0\}$ & prioritization strength
& Eq.~\eqref{eq:prob_weight_appendix} \\
\bottomrule
\end{tabular}
\caption{Hyperparameters for dataset-level replay scheduling (Appendix~\ref{app:replay_derivation}).}
\label{tab:replay_hparams}
\end{table}

\section{Training Dataset Details}
\label{app:training_datasets}

We describe the datasets used in our sequential fine-tuning experiments, 
which span instruction following, natural language understanding, 
and progressively harder quantitative reasoning tasks.
Together, they form a curriculum from general-purpose instruction learning 
to long-horizon mathematical reasoning.

\paragraph{AG News.}
AG News~\citep{zhang2015character} is a topic classification dataset consisting of news articles labeled into four categories:
\textit{World}, \textit{Sports}, \textit{Business}, and \textit{Sci/Tech}.
We use the standard training split and evaluate performance using exact-match accuracy.
Each example is formatted as a short text classification prompt, asking the model to predict the corresponding topic label.
This dataset introduces a lightweight supervised classification stage within the continual learning sequence.

\paragraph{SQuAD.}
The Stanford Question Answering Dataset (SQuAD)~\citep{rajpurkar2016squad} is an extractive reading comprehension benchmark,
where each question requires identifying a contiguous answer span from a given passage.
We evaluate performance using the standard token-level F1 score following the official evaluation protocol.
This dataset complements classification-style tasks by introducing span-based reading comprehension.

\paragraph{SciQ.}
SciQ~\citep{welbl2017crowdsourcing} is a multiple-choice science question answering dataset,
containing approximately 13.7K questions constructed via crowdsourcing and web-based extraction.
Each instance consists of a question, four answer options, and a single correct choice.
We convert each example into a multiple-choice prompt and supervise the model to generate the correct option.
Performance is evaluated using exact-match accuracy over answer choices.

\paragraph{BoolQ.}
BoolQ~\citep{clark2019boolq} is a binary (yes/no) question answering dataset derived from real user queries paired with supporting passages.
Each instance requires the model to determine whether a given statement is true or false based on the provided context.
We formulate BoolQ as a binary classification task and evaluate using exact-match accuracy.
This dataset emphasizes logical reasoning and natural language inference.

\paragraph{ARC.}
The AI2 Reasoning Challenge (ARC)~\citep{clark2018think} is a multiple-choice science question answering benchmark,
designed to test elementary-level reasoning beyond simple retrieval.
We use the challenge setting, where each question has four or five candidate answers.
All examples are formatted as multiple-choice prompts, and performance is measured by exact-match accuracy.
ARC introduces more challenging reasoning scenarios with relatively low pre-training familiarity.

\paragraph{GSM8K (reasoning-formatted).}
GSM8K~\citep{cobbe2021gsm8k} contains around 8.5K grade-school mathematical word problems.
We use a reasoning-formatted variant where each example includes a question and a step-by-step solution (rationale) followed by the final answer.
We apply light text normalization (whitespace and punctuation), symbol standardization, and the same tokenizer as the base LLM.
This stage strengthens structured numerical reasoning and provides a controlled transition from general language instruction to math-focused training.

\paragraph{Competition Math (MATH).}
Competition Math~\citep{hendrycksmath2021} (commonly referred to as the MATH dataset) consists of approximately 12K problems drawn from math competitions (e.g., AMC/AIME and Olympiad-style exams),
covering algebra, geometry, combinatorics, and number theory.
We use the short-answer setting, where each problem is paired with its final numeric or symbolic solution.
For consistency with earlier stages, we convert each instance into a reasoning-style prompt--response format when intermediate solutions are available; otherwise, we supervise only the final answer.

\paragraph{Data splits and preprocessing.}
For each dataset, we randomly split 95\% of the samples for training and 5\% for validation using a fixed random seed.
Preprocessing includes removing non-semantic markup (e.g., stray LaTeX environments) while preserving mathematical content,
standardizing numeric formats, and filtering invalid or empty answers.
All datasets are tokenized with the base LLM tokenizer to maintain vocabulary consistency.
We cap the maximum sequence length at 2048 tokens; for overlength examples, we truncate from the end, prioritizing retention of the question, reasoning steps (when present), and the final answer.

\section{Evaluation Metric Definitions}
\label{app:eval_metric}

This section defines the evaluation metrics used in our experiments: BLEU, ROUGE-L, exact-match (EM) accuracy, and the average normalized score.

\paragraph{BLEU.}
Following \citet{papineni2002bleu}, BLEU measures modified $n$-gram precision between a generated response $G$ and its reference $R$:
\begin{equation}
\mathrm{BLEU} \;=\; \mathrm{BP}\cdot \exp\!\left(\sum_{n=1}^{n_{\max}} w_n \log p_n\right),
\end{equation}
where $p_n$ is the modified $n$-gram precision, $w_n$ are uniform weights ($w_n=1/n_{\max}$), and $\mathrm{BP}$ is the brevity penalty.
Unless otherwise noted, we report BLEU-4 with smoothing enabled.
For reproducibility, we compute BLEU using \texttt{sacreBLEU} with standard tokenization and smoothing (details in the released code/config).

\paragraph{ROUGE-L.}
ROUGE-L~\citep{lin2004rouge} is based on the longest common subsequence (LCS) between $G$ and $R$.
Let $L=\mathrm{LCS}(G,R)$, then
\begin{equation}
P_{\mathrm{LCS}}=\frac{L}{|G|},\qquad
R_{\mathrm{LCS}}=\frac{L}{|R|},
\end{equation}
and the ROUGE-L F-score is
\begin{equation}
\mathrm{ROUGE\mbox{-}L}
=\frac{(1+\beta^2)\,P_{\mathrm{LCS}}\,R_{\mathrm{LCS}}}{R_{\mathrm{LCS}}+\beta^2 P_{\mathrm{LCS}}}.
\end{equation}
We set $\beta=1.2$ following common practice and compute ROUGE-L on tokenized text with the same preprocessing pipeline across methods.

\paragraph{Exact-match (EM) accuracy.}
For math reasoning datasets (\textbf{GSM8K} and \textbf{Competition Math/MATH}), we compute exact-match accuracy as
\begin{equation}
\mathrm{EM}
=\frac{1}{M}\sum_{i=1}^{M}\mathbf{1}\!\left[\mathrm{Norm}(\hat{y}_i)=\mathrm{Norm}(y_i)\right],
\end{equation}
where $\hat{y}_i$ is the model-predicted \emph{final answer}, $y_i$ is the ground truth, and $M$ is the number of evaluated samples.
$\mathrm{Norm}(\cdot)$ denotes a deterministic answer-normalization procedure (e.g., stripping whitespace and punctuation, removing wrappers such as ``\texttt{\textbackslash boxed\{\}}'', and canonicalizing numeric formats).
This metric evaluates end-to-end correctness of the final answer regardless of intermediate reasoning steps.

\paragraph{Average normalized score.}
To summarize overall stability across datasets, we compute a per-dataset normalized score
\begin{equation}
\tilde{s}_d=\frac{s_d}{\max(s_d^{\max},\epsilon)},
\end{equation}
where $s_d$ is the raw metric on dataset $d$ after the \emph{final} training stage (BLEU/ROUGE-L for instruction data, EM for math data),
and $s_d^{\max}$ is the maximum score achieved by the \emph{same training run} across all stages when evaluated on dataset $d$.
We use $\epsilon$ (a small constant) for numerical safety.
The average normalized score is
\begin{equation}
\mathrm{Score}_{\mathrm{avg}}=\frac{1}{D}\sum_{d=1}^{D}\tilde{s}_d,
\end{equation}
where $D=3$ is the number of datasets.
A higher value indicates better overall retention and cross-domain stability.

\section{Training and Hyperparameter Settings}
\label{app:training_details}

This appendix summarizes the optimization configuration shared by all methods and clarifies practical settings used by the replay implementation.
Model-agnostic replay/memory hyperparameters are reported in Appendix~\ref{app:sample_level_derivations} (Table~\ref{tab:memory_hparams}) and Appendix~\ref{app:replay_derivation} (Table~\ref{tab:replay_hparams}); here we focus on training setup and engineering details required for reproducibility.

\subsection{General Fine-tuning Configuration}
\label{app:training_general}

All methods are implemented in \textbf{LLaMAFactory}~\citep{zheng2024llamafactory} and evaluated on three pretrained backbones:
\textbf{Gemma2-9B}, \textbf{Qwen2.5-7B}, and \textbf{LLaMA-3.1-8B} (Appendix~\ref{app:model_details}).
Unless otherwise stated, we keep the optimizer, LoRA configuration, training steps, and evaluation protocol identical across backbones and replay strategies.
Because GPU memory footprints differ across backbones, the \emph{micro-batch size} may vary; we therefore adjust gradient accumulation to keep the \emph{effective} batch size fixed.

\begin{table}[t]
\centering
\caption{LoRA fine-tuning and optimization hyperparameters (shared across replay strategies).}
\label{tab:training_hyperparams}
\small
\begin{tabular}{ll}
\toprule
\textbf{Item} & \textbf{Setting} \\
\midrule
Backbone model & Gemma2-9B / Qwen2.5-7B / LLaMA-3.1-8B \\
Fine-tuning framework & LLaMAFactory~\citep{zheng2024llamafactory} \\
LoRA rank $r$ & 8 \\
LoRA scaling $\alpha_{\mathrm{lora}}$ & 16 \\
LoRA target modules & \texttt{q\_proj}, \texttt{v\_proj} (all Transformer layers) \\
Optimizer & AdamW \\
Learning rate & $2\times10^{-4}$ \\
Adam $\beta_1,\beta_2$ & (0.9, 0.999) \\
Weight decay & 0.01 \\
LR schedule & cosine decay with warmup ratio 0.05 \\
Max sequence length & 2048 \\
Micro-batch size (per GPU) & backbone-dependent (memory-limited) \\
Gradient accumulation & adjusted to match effective batch \\
Effective batch size & fixed across backbones (256 sequences) \\
Training steps (per stage) & 2000 \\
Evaluation frequency & every 100 steps \\
Precision & bfloat16 (fp16 fallback if bf16 unsupported) \\
\bottomrule
\end{tabular}
\end{table}

\subsection{Replay Implementation Details}
\label{app:training_replay_impl}

Dataset-level scheduling uses $(\Delta t_r^{(1)}, \eta_p, \rho_p)$ for interval expansion, $(\lambda_0,\lambda_{\min},\beta_r)$ for replay ratio decay, and $\zeta$ for prioritization (Table~\ref{tab:replay_hparams}).
Sample-level memory dynamics follow the hazard and consolidation parameters in Table~\ref{tab:memory_hparams}.
The following additional settings control replay overhead and buffering:

\begin{table}[t]
\centering
\caption{Practical replay implementation settings.}
\label{tab:replay_impl_params}
\small
\begin{tabular}{ll}
\toprule
\textbf{Item} & \textbf{Setting} \\
\midrule
Replay buffer $B_{\mathrm{replay}}$ size & 1024 samples (maintained within each stage) \\
Buffer refresh interval & every 50 training steps \\
Replay batch composition & mix current-task samples with replayed samples at ratio $\lambda_t$ \\
Replay sampling distribution & $p_i^{(t_k)} \propto m_{i,t_k}^{-\zeta}$ (Eq.~\eqref{eq:prob_weight_appendix}) \\
Lazy state update & compute $m_{i,t}$ on demand using the latest epoch-/cycle-level hazard estimate (App.~\ref{app:epoch}) \\
Staleness safety cap & force refresh if $\Delta t_i$ exceeds a fixed cap (default: 10 epochs) \\
\bottomrule
\end{tabular}
\end{table}

\subsection{Reproducibility Notes}
\label{app:training_repro}

All runs use a fixed random seed (42) for data shuffling, replay-buffer initialization, and sampling.
We train with PyTorch 2.x under distributed data-parallel (DDP) with CUDA.
Data loading uses 8 workers per GPU.
We log training curves, per-stage evaluation metrics, and replay statistics (e.g., $\lambda_t$, $\Delta t_r^{(k)}$, and summary statistics of $m_{i,t}$) for all methods.

\paragraph{Sequence truncation.}
All inputs are truncated to the maximum sequence length (2048 tokens).
When truncation is required, we prioritize preserving the prompt and the final answer segment, and remove trailing tokens first.

\section{Model Details}
\label{app:model_details}

We employ three open-source large language models (LLMs) as pretrained backbones for all experiments:
Gemma2-9B~\citep{team2024gemma}, LLaMA-3.1-8B~\citep{dubey2024llama}, and Qwen2.5-7B~\citep{yang2024qwen2}.
All models are obtained from the Hugging Face Hub under their respective licenses and are used in their \emph{pretrained} (non--instruction-tuned) variants to minimize prior exposure to our evaluation tasks.

\paragraph{Gemma2-9B.}
Gemma2-9B is a decoder-only Transformer model with approximately 9B parameters.
We use the official pretrained checkpoint and tokenizer released by Google.
While each backbone supports a large native context window, in our experiments we cap the maximum sequence length to 2048 tokens for fair comparison across backbones and to match the sequential fine-tuning setup.

\paragraph{LLaMA-3.1-8B.}
LLaMA-3.1-8B is a decoder-only Transformer with 8B parameters.
We use the Hugging Face checkpoint (\texttt{meta-llama/Llama-3.1-8B}) and its default tokenizer.
As above, we set the effective training/evaluation maximum sequence length to 2048 tokens for consistency across backbones.

\paragraph{Qwen2.5-7B.}
Qwen2.5-7B is a 7B-parameter pretrained model with strong multilingual coverage and long-context capability.
We use the base checkpoint (\texttt{Qwen/Qwen2.5-7B}) and the official tokenizer.
This backbone is used for our primary ablations and analysis due to its strong reasoning performance under the same fine-tuning budget.

\paragraph{Implementation and fine-tuning setup.}
All backbones are fine-tuned using \textbf{LLaMAFactory}~\citep{zheng2024llamafactory} with parameter-efficient LoRA~\citep{hu2022lora}.
We freeze all backbone weights and train rank-$r{=}8$ LoRA adapters inserted into the attention projection layers (\texttt{q\_proj} and \texttt{v\_proj}) for all Transformer blocks.
Optimization uses AdamW with learning rate $2\times10^{-4}$, warmup ratio $0.05$, and weight decay $0.01$.
We use bfloat16 training when supported (fp16 otherwise).
Due to different memory footprints across backbones, micro-batch sizes may vary; we adjust gradient accumulation to keep the \emph{effective} batch size fixed across backbones (Appendix~\ref{app:training_details}).
For reproducibility, we fix random seeds to 42 and keep data ordering identical across model runs.